\documentclass{article}

\usepackage{microtype}
\usepackage{graphicx}
\usepackage{subfigure}
\usepackage{booktabs} %

\usepackage{hyperref}
\usepackage{bm}

\usepackage[accepted]{icml2025}

\usepackage{amsmath}
\usepackage{enumitem}
\usepackage{amssymb}
\usepackage{mathtools}
\usepackage{amsthm}
\usepackage{ stmaryrd }

\usepackage{multirow}

\newcommand{\bff}{\mathbf{f}}

\newcommand{\bg}{\mathbf{g}}

\newcommand{\bu}{\mathbf{u}}

\newcommand{\R}{\mathbb{R}}

\usepackage[capitalize,noabbrev]{cleveref}

\usepackage[colorinlistoftodos]{todonotes}

\usepackage{dsfont}

\theoremstyle{plain}
\newtheorem{theorem}{Theorem}[section]
\newtheorem{proposition}[theorem]{Proposition}
\newtheorem{lemma}[theorem]{Lemma}

\theoremstyle{definition}
\newtheorem{definition}[theorem]{Definition}

\theoremstyle{remark}
\newtheorem{remark}[theorem]{Remark}

\newcommand{\OTCP}{\textbf{OT-CP}}
\newcommand{\MCP}{\textbf{M-CP}}
\newcommand{\MergeCP}{\textbf{Merge-CP}}

\newcommand{\argmin}{\mathop{\mathrm{arg\,min}}}
\DeclareMathOperator{\lse}{min}
\DeclareMathOperator*{\argmax}{argmax}
\graphicspath{{figures/}}

\icmltitlerunning{Multivariate Conformal Prediction using Optimal Transport}

\begin{document}
\twocolumn[
\icmltitle{Multivariate Conformal Prediction using Optimal Transport}

\icmlsetsymbol{equal}{*}

\begin{icmlauthorlist}

\icmlauthor{Michal Klein}{yyy}
\icmlauthor{Louis Bethune}{yyy}
\icmlauthor{Eugene Ndiaye}{yyy}
\icmlauthor{Marco Cuturi}{yyy}
\icmlcorrespondingauthor{Eugene Ndiaye}{e\_ndiaye@apple.com}
\end{icmlauthorlist}

\icmlaffiliation{yyy}{Apple}

\icmlkeywords{Machine Learning, ICML}

\vskip 0.3in
]

\printAffiliationsAndNotice{\icmlEqualContribution} %

\begin{abstract}
Conformal prediction (CP) quantifies the uncertainty of machine learning models by constructing sets of plausible outputs. These sets are constructed by leveraging a so-called conformity score, a quantity computed using the input point of interest, a prediction model, and past observations. CP sets are then obtained by evaluating the conformity score of all possible outputs, and selecting them according to the rank of their scores. 
Due to this ranking step, most CP approaches rely on a score functions that are univariate. The challenge in extending these scores to multivariate spaces lies in the fact that no canonical order for vectors exists. To address this, we leverage a natural extension of multivariate score ranking based on optimal transport (OT). Our method, \OTCP, offers a principled framework for constructing conformal prediction sets in multidimensional settings, preserving distribution-free coverage guarantees with finite data samples. We demonstrate tangible gains in a benchmark dataset of multivariate regression problems and address computational \& statistical trade-offs that arise when estimating conformity scores through OT maps.
\end{abstract}

\section{Introduction}

Conformal prediction (CP) \cite{gammerman1998learning, Vovk_Gammerman_Shafer05, Shafer_Vovk08} has emerged as a simple framework to quantify the prediction uncertainty of machine learning algorithms without relying on distributional assumptions on the data. For a sequence of observed data, and a new input point,
$$D_n = \{(x_1, y_1), ..., (x_n, y_n)\} \text{ and } x_{n+1},$$
the objective is to construct a set that contains the unobserved response $y_{n+1}$ with a specified confidence level $100(1- \alpha)\%$. This involves evaluating scores $S(x, y, \hat y)\in\mathbb{R}$ such as the prediction error of a model $\hat y$, for each observation $(x, y)$ in $D_n$ and ranking these score values. The conformal prediction set for the new input $x_{n+1}$ is the collection of all possible responses $y$ whose score $S(x_{n+1}, y, \hat y)$ ranks small enough to meet the prescribed confidence threshold, compared to the scores $S(x_i, y_i, \hat y)$ in the observed data. 

CP has undergone tremendous developments in recent years ~\citep{barber2023conformal,park2024semiparametric,tibshirani2019conformal, guha2024conformal} which mirror its increased applicability to challenging settings \citep{straitouri2023improving,lu2022fair}. To name a few, it has been applied for designing uncertainty sets in active learning \citep{Ho_Wechsler08}, anomaly detection \citep{Laxhammar_Falkman15, Bates_Candes_Lei_Romano_Sesia21}, few-shot learning \citep{Fisch_Schuster_Jaakkola_Barzilay21}, time series \citep{Chernozhukov_Wuthrich_Zhu18, Xu_Xie20, chernozhukov2021exact, Lin_Trivedi_Sun22, zaffran2022adaptive}, or to infer performance guarantees for statistical learning algorithms \citep{Holland20, Cella_Martin20}; and recently to Large Language Models \cite{kumar2023conformal, quach2023conformal}. We refer to the extensive reviews in \citep{balasubramanian2014conformal} for other applications in machine learning.  \looseness=-1

By design, CP requires the notion of order, as the inclusion of a candidate response depends on its relative ranking to the scores observed previously. Hence, the classical strategies developed so far largely target score functions with univariate outputs. This limits their applicability to multivariate responses, as ranking multivariate scores $S(x, y, \hat y) \in \mathbb{R}^d, d\geq 2$ is not as straightforward as ranking univariate scores in $\mathbb{R}$. 

\textbf{Ordering Vector Distributions using Optimal Transport.} In parallel to these developments, and starting with the seminal reference of \citep{chernozhukov2017} and more generally the pioneering work of \citep{hallin2021,hallin2022center,hallin2023efficient}, multiple references have explored the possibilities offered by the optimal transport theory to define a meaningful ranking or ordering in a multidimensional space. Simply put, the analog of a rank function computed on the data can be found in the optimal \citeauthor{Bre91} map that transports the data measure to a uniform, symmetric, centered measure of reference in $\mathbb{R}^d$. As a result, a simple notion of a univariate rank for a vector $z\in\mathbb{R}^d$ can be found by evaluating the distance of the image of $z$ (according to that optimal map) to the origin. This approach ensures that the ordering respects both the geometry, i.e., the spatial arrangement of the data and its distribution: points closer to the center get lower ranks.\looseness=-1

\paragraph{Contributions}
We propose to leverage recent advances in computational optimal transport~\citep{PeyCut19}, using notably differentiable transport map estimators~\citep{pooladian2021entropic,cuturi2019differentiable}, and apply such map estimators in the definition of multivariate score functions. More precisely:
\begin{itemize}[leftmargin=.2cm,itemsep=.0cm,topsep=0cm,parsep=2pt]
\item \OTCP: We extend conformal prediction techniques to multivariate score functions by leveraging optimal transport ordering, which offers a principled way to define and compute a higher-dimensional quantile and cumulative distribution function. As a result, we obtain distribution-free uncertainty sets that capture the joint behavior of multivariate predictions that enhance the flexibility and scope of conformal predictions.
\item We propose a computational approach to this theoretical ansatz using the entropic map~\citep{pooladian2021entropic} computed from solutions to the \citeauthor{Sinkhorn64} problem~\citep{cuturi2013sinkhorn}. We prove that our approach preserves the coverage guarantee while being tractable.
\item We show the application of \OTCP\ using a recently released benchmark of regression tasks~\citep{dheur2025multioutputconformalregressionunified}.
\end{itemize}
We acknowledge the concurrent proposal of \citet{thurin2025optimaltransportbasedconformalprediction}, who adopt a similar approach to ours, with, however, a few important practical differences, discussed in more detail in Section~\ref{sec:concurrent}.

\section{Background}

\subsection{Univariate Conformal Prediction}
\label{subsec:Univariate_Conformal_Prediction}

We recall the basics of conformal prediction based on real-valued score function and refer to the recent tutorials \citep{Shafer_Vovk08, angelopoulos2021gentle}. In the following, we denote $[n]:=\{1,\dots, n\}$.

For a real-valued random variable $Z$, it is common to construct an interval $[a,b]$, within which it is expected to fall, as\looseness=-1
\begin{equation}\label{eq:true_confidence_set}
\mathcal{R}_\alpha = \{z \in \mathbb{R}: F(z) \in [a, b]\}
\end{equation}
This is based on the probability integral transform that states that the cumulative distribution function $F$ maps variables to uniform distribution, i.e., 
$\mathbb{P}(F(Z) \in [a, b]) = \mathbb{U}([a, b]).$
To guarantee a $(1-\alpha)$ uncertainty region, it suffices to choose $a$ and $b$ such that $ \mathbb{U}([a, b]) \geq 1-\alpha$ which implies
\begin{equation}\label{eq:coverage_exact_univariate_quantile_region}
    \mathbb{P}\left(Z \in \mathcal{R}_\alpha\right) \geq 1-\alpha.
\end{equation}
Applying it to a real-valued score $Z = S(X, Y)$ of the prediction model $\hat y$, an uncertainty set for the response of a given a input $X$ can be expressed as
\begin{equation}\label{eq:exact_uq}
\mathcal{R}_\alpha(X) = \big\{y \in \mathcal{Y}: F\circ S(X, y) \in [a, b]\big\}.
\end{equation}
However, this result is typically not directly usable, as the ground-truth distribution $F$ is unknown and must be approximated empirically with $F_n$ using finite samples of data. When the sample size goes to infinity, one expects to recover \Cref{eq:coverage_exact_univariate_quantile_region}.
The following result provides the tool to obtain the finite sample version \cite{Shafer_Vovk08}.

\begin{lemma}\label{lm:PIT_onedim}
If $Z_1, \dots, Z_n, Z$ be a sequence of real-valued exchangeable random variables, then it holds
\begin{align*}
 F_n(Z) &\sim \mathbb{U}\left\{0, \frac1n, \frac2n, \ldots, 1\right\} \\
\mathbb{P}(F_n(Z) \in [a, b]) &= \mathbb{U}_{n+1}([a, b])
= \frac{\lfloor n b\rfloor - \lceil n a\rceil + 1}{n+1} .
\end{align*}
\end{lemma}
By choosing any $a, b$ such that $\mathbb{U}_{n+1}([a, b]) \geq 1 - \alpha$, \Cref{lm:PIT_onedim} guarantees  a coverage, that is at least equal to the prescribed level of uncertainty
\begin{equation*}
\mathbb{P}\left(Z \in \mathcal{R}_{\alpha, n}\right) \geq 1 - \alpha.
\end{equation*}
where, the uncertainty set $\mathcal{R}_{\alpha, n} = \mathcal{R}_{\alpha}(D_n)$ is defined based on observations $D_n = \{Z_1, \ldots, Z_n\}$ as:
\begin{equation}\label{eq:empirical_uq}
    \mathcal{R}_{\alpha, n} = \big\{z \in \mathbb{R} : F_n(z) \in [a, b]\big\} .
\end{equation}

In short, \Cref{eq:empirical_uq} is an empirical version of \Cref{eq:true_confidence_set} based on finite data samples that still preserves the coverage probability $(1-\alpha)$ and does not depend on the ground-truth distribution of the data.

Given data $D_n$,
a prediction model $\hat y$ and a new input $X_{n+1}$, one can build an uncertainty set for the unobserved output $Y_{n+1}$ by applying it to observed score functions. 

\begin{proposition}[Conformal Prediction Coverage]\label{prop:Univariate_Conformal_prediction_Coverage}
Consider $Z_i = S(X_i, Y_i)$ for $i$ in $[n]$ and $Z=S(X_{n+1}, Y_{n+1})$ in \Cref{lm:PIT_onedim}.
The conformal prediction set is defined as
\begin{align*}
\mathcal{R}_{\alpha, n}(X_{n+1}) &= \big\{y \in \mathcal{Y} : F_n \circ S(X_{n+1}, y) \in [a, b]\big\} 
\end{align*}
and satisfies a finite sample coverage guarantee
$$
\mathbb{P}\left(Y_{n+1} \in \mathcal{R}_{\alpha, n}(X_{n+1})\right) \geq 1 - \alpha.
$$
\end{proposition}
The conformal prediction coverage guarantee in \Cref{prop:Univariate_Conformal_prediction_Coverage} holds for the \emph{unknown} ground-truth distribution of the data $\mathbb{P}$, does not require quantifying the estimation error $|F_n - F|$, and is applicable to any prediction model $\hat y$ as long as it treats the data exchangeably, e.g., a pre-trained model independent of $D_n$.

Leveraging the quantile function $F_{n}^{-1}= Q_n$, and
by setting $a=0$ and $b=1-\alpha$, we have the usual description
\begin{align*}
\mathcal{R}_{\alpha, n}(X_{n+1}) = \big\{y \in \mathcal{Y} : S(X_{n+1}, y) \leq Q_n(1-\alpha) \big\} 
\end{align*}
namely the set of all possible responses whose score rank is smaller or equal to $\lceil (1-\alpha)(n+1) \rceil$ compared to the rankings of previously observed scores. For the absolute value difference score function, the CP set corresponds to 
$$\mathcal{R}_{\alpha, n}(X_{n+1}) = \big[\hat y(X_{n+1}) \pm Q_n(1-\alpha)\big].$$

\paragraph{Center-Outward View}
Another classical choice is $a=\frac{\alpha}{2}$ and $b=1-\frac{\alpha}{2}$. In that case, we have the usual confidence set that corresponds to a range of values that captures the central proportion with $\alpha/2$ of the data lying below $Q(\alpha/2)$ and $\alpha/2$ lying above $Q(1-\alpha/2)$.

Introducing the center-outward distribution of $Z$ as the function $T = 2 F - 1$ , the probability integral transform $T(Z)$ is uniform in the unit ball $[-1, 1]$.
This ensures a symmetric description of 
$
\mathcal{R}_\alpha = T^{-1}(B(0, 1-\alpha))$ around a central point such as the median $Q(1/2) = T^{-1}(0)$,
with the radius of the ball that corresponds to the desired confidence level of uncertainty. Similarly, we have the empirical center-outward distribution $T_{n} = 2 F_n - 1$ and
the center-outward view of the conformal prediction set follows as
\begin{align*}
\mathcal{R}_{\alpha, n}(X_{n+1}) &= \big\{y \in \mathcal{Y} : |T_{n} \circ S(X_{n+1}, y)| \leq 1-\alpha \big\} .
\end{align*}

If $Z$ follows a probability distribution $\mathbb{P}$, then the transformation $z \mapsto T(z)$ is mapping the source  distribution $\mathbb{P}$ to the uniform distribution $\mathbb{U}$ over a unit ball. In fact, it can be characterized as essentially the unique monotone increasing function such that $T(Z)$ is uniformly distributed. 

\subsection{Multivariate Conformal Prediction}
While many conformal methods exist for univariate prediction, we focus here on those applicable to \textit{multivariate} outputs.  As recalled in~\citep{dheur2025multioutputconformalregressionunified}, several alternative conformal prediction approaches have been proposed to tackle multivariate prediction problems. Some of these methods can directly operate using a simple predictor (e.g., a conditional mean) of the response $y$, while some may require stronger assumptions, such as requiring an estimator of the \textit{joint} probability density function between $x$ and $y$, or access to a generative model that mimics the \textit{conditional} distribution of $y$ given $x$) \cite{izbicki2022cd, wang2022probabilistic}. \looseness=-1

We restrict our attention to approaches that make no such assumption, reflecting our modeling choices for \OTCP. 

\MCP.
We will consider the template approach of \cite{zhou2024conformalized} to use classical CP by aggregating a score function computed on each of the $d$ outputs of the multivariate response. Given a conformity score \( s_i \) (to be defined next) for the \( i \)-th dimension, \citet{zhou2024conformalized} define the following aggregation rule:
\begin{equation}
    s_{\text{M-CP}}(x, y) = \max_{i \in [d]} s_i(x, y_i).
    \label{eq:simultaneous_score}
\end{equation}
As~\citep{dheur2025multioutputconformalregressionunified}, we will use \textit{conformalized quantile regression} \citep{romano2019conformalized} to define the score functions above, for each output \( i \in [d] \), where the conformity score is given by:
\[
s_i(x, y_i) = \max\{\hat{l}_i(x) - y_i, y_i - \hat{u}_i(x)\},
\]
with \( \hat{l}_i(x) \) and \( \hat{u}_i(x) \) representing the lower and upper conditional quantiles of \( Y_i|X=x\) at levels \( \alpha_l \) and \( \alpha_u \), respectively. In our experiments, we consider equal-tailed prediction intervals, where \( \alpha_l = \frac{\alpha}{2} \), \( \alpha_u = 1 - \frac{\alpha}{2} \), and \( \alpha \) denotes the miscoverage level. 

\MergeCP. An alternative approach is simply to use a squared Euclidean aggregation, 
$$
s(x, y) :=  \|\hat{y}(x) - y\|_2,
$$
where the choice of the norm (e.g., $\ell_1$, $\ell_2$, or $\ell_\infty$) depends on the desired sensitivity to errors across tasks. This approach reduces the multidimensional residual to a scalar conformity score, leveraging the natural ordering of real numbers. This simplification not only makes it straightforward to apply univariate conformal prediction methods, but also avoids the complexities of directly managing vector-valued scores in conformal prediction. A variant consists of applying a Mahalanobis norm~\citep{pmlr-v152-johnstone21a} in lieu of the squared Euclidean norm, using the covariance matrix $\Sigma$ estimated from the training data \citep{pmlr-v152-johnstone21a,pmlr-v230-katsios24a, henderson2024adaptive}, 
$$
s(x, y) :=  \|\Sigma^{-1/2}(\hat{y}(x) - y)\|_2,
$$

\subsection{Kantorovich Ranks}
A naive way to define ranks in multiple dimensions might be to measure how far each point is from the origin and then rank them by that distance. This breaks down if the distribution of the data is stretched or skewed in certain directions. To correct for this, \citet{hallin2021} developed a formal framework of center-outward distributions and quantiles, also called Kantorovich ranks~\citep{chernozhukov2017}, extending the familiar univariate concepts of ranks and quantiles into higher dimensions by building on elements of optimal transport theory.

\paragraph{Optimal Transport Map.} Let $\mu$  and $\nu$ be source and target probability measures on $\Omega \subset \mathbb{R}^d$. One can look for a map $T: \Omega \to \Omega$ that pushes forward $\mu$ to $\nu$ and minimizes the average transportation cost
\begin{equation}\label{eq:brenier}
T^\star \in \argmin_{T_\# \mu = \nu} \int_{\Omega} \|x - T(x)\|^2 \, d\mu(x).
\end{equation}
\citeauthor{Bre91}’s theorem states that if the source measure $\mu$ has a density, there exists a solution to \eqref{eq:brenier} that is the gradient of a convex function $\phi: \Omega \to \mathbb{R}$ such that $T^\star = \nabla \phi$.

In the one-dimensional case, the cumulative distribution function of a distribution $\mathbb{P}$ is the unique increasing function transporting it to the uniform distribution. This monotonicity property generalizes to higher dimensions through the gradient of a convex function $\nabla \phi$. Thus, one may view the optimal transport map in higher dimensions as a natural analog of the univariate cumulative distribution function: both represent a unique, monotone way to send one probability distribution onto another.

\begin{definition}
    The center-outward distribution of a random variable $Z \sim \mathbb{P}$ is defined as the optimal transport map $T = \nabla \phi$ that pushes $\mathbb{P}$ forward to the uniform distribution $\mathbb{U}$ on the unit ball $B(0,1)$.
    The rank of $Z$ is defined as $\mathrm{Rank}(Z) = \|T(Z)\|$, the distance from the origin.
\end{definition}

\paragraph{Quantile region} is an extension of quantiles to multiple dimensions to represent region in the sample space that contains a given proportion of probability mass. The quantile region at probability level $(1-\alpha) \in (0,1)$ can be defined as\looseness=-1
$$
\mathcal{R}_{\alpha} = \{z \in \mathbb{R}^d : \|T(z)\| \leq  1-\alpha\}.
$$
By definition of the spherical uniform distribution, we have $\|T(Z)\|$ is uniform on $(0,1)$ which implies
\begin{align}\label{eq:exact_continuous_coverage}
\mathbb{P}(Z \in \mathcal{R}_{\alpha}) = 1-\alpha.
\end{align}

\subsection{Entropic Map.}
A convenient estimator to approximate the \citeauthor{Bre91} map $T^\star$ from samples $(z_1,\dots, z_n)$ and $(u_1, \dots, u_m)$ is the entropic map~\citep{pooladian2021entropic}: Let $\varepsilon>0$ and write $K_{ij} = [\exp(-\|z_i-u_j\|^2/\varepsilon)]_{ij}$, the kernel matrix. Define,
\begin{equation}\label{eq:finitedual}
\!\!\!\bff^\star, \bg^\star = \argmax_{\bff\in\R^n,\bg\in\R^m} \langle\bff, \tfrac{\mathbf{1}_n}{n}\rangle + \langle\bg, \tfrac{\mathbf{1}_m}{m}\rangle  - \varepsilon \langle e^{\frac{\bff}{\varepsilon}}, K e^{\frac{\bg}{\varepsilon}}\rangle\,.
\end{equation}
The \Cref{eq:finitedual} is an unconstrained concave optimization problem known as the regularized OT problem in dual form~\citep[Prop.~4.4]{PeyCut19} and can be solved numerically with the \citeauthor{Sinkhorn64} algorithm~\citep{cuturi2013sinkhorn}. Equipped with these optimal vectors, one can define the maps, valid out of sample:
\begin{align}\label{eq:fdual}
f_\varepsilon(z)= \lse_\varepsilon([\|z-u_j\|^2 - \bg^\star_j]_j)\,,\\
\label{eq:gdual}
g_\varepsilon(u)= \lse_\varepsilon([\|z_i-u\|^2 - \bff^\star_i]_i)\,,
\end{align}
where for a vector $\bu$ or arbitrary size $s$ we define the log-sum-exp operator as $\lse_\varepsilon(\bu):= - \varepsilon \log (\tfrac{1}{s}\mathbf{1}_s^Te^{-\bu/\varepsilon})$. Using the \citet{Bre91} theorem, linking potential values to optimal map estimation, one obtains an estimator for $T^\star$:
\begin{equation}\label{eq:empirical_entropic_map}
T_{\varepsilon}(z) : = z - \nabla f_\varepsilon(z) = \sum_{j=1}^m p_{\,j}(z)u_j\,,
\end{equation}
where the weights depend on $z$ as:
\begin{equation}\label{eq:gibbs}
p_{\,j}(z):=\frac{\exp\left(- \left(\|z-u_j\|^2-\bg_j^\star\right)/\varepsilon\right)}{\sum_{k=1}^m \exp\left(- \left(\|z-u_k\|^2-\bg_k^\star\right)/\varepsilon\right)}\,.
\end{equation} 
Analogously to \eqref{eq:gibbs}, one can obtain an estimator for the inverse map $(T^\star)^{-1}$ as $T^{\text{inv}}_{\varepsilon}(u) : = \sum_{i=1}^n q_{\,j}(u)z_j\,,$ with weights $q_{\,j}(u)$ arising for a vector $u$ from the Gibbs distribution of the values $[\|z_i-u\|^2 - \bff^\star_i]_i$

\section{Kantorovich Conformal Prediction }

\subsection{Multi-Output Conformal Prediction}
We suppose that $\mathbb{P}$ is only available through a finite samples and consider the \textit{discrete} transport map
$$
T_{n+1} : (Z_i)_{i \in [n+1]} \to (U_i)_{i \in [n+1]}
$$
which can be obtained by solving the optimal assignment problem, which seeks to minimize the total transport cost between the empirical distributions $\mathbb{P}_{n+1}$ and $\mathbb{U}_{n+1}$:
\begin{equation}\label{eq:empirical_transport_map}
T_{n+1} \in \argmin_{T \in \mathcal{T}} \sum_{i=1}^{n+1} \|Z_i - T(Z_i)\|^2,
\end{equation}
where $\mathcal{T}$ is the set of bijections mapping the observed sample $(Z_i)_{i \in [n+1]}$ to the target grid $(U_i)_{i \in [n+1]}$. 

\begin{definition}    
    Let $(Z_1, \ldots, Z_n, Z_{n+1})$ be a sequence of exchangeable variables in $\mathbb{R}^d$ that follow a common distribution $\mathbb{P}$. The discrete center-outward distribution $T_{n+1}$ is the transport map pushing forward $\mathbb{P}_{n+1}$ to $\mathbb{U}_{n+1}$.
\end{definition}

Following \cite{hallin2021}, we begin by constructing the target discritbution $\mathbb{U}_{n+1}$ as a discretized version of a spherical uniform distribution. It is defined such that the total number of points $n + 1 = n_R n_S + n_o$, where $n_o$ points are at the origin:\looseness=-1
\begin{itemize}
    \item $n_S$ unit vectors $\mathbf{u}_1, \ldots, \mathbf{u}_{n_S}$ are uniform on the sphere.
    \item $n_R$ radius are regularly spaced as $\left\{\frac{1}{n_R}, \frac{2}{n_R}, \ldots, 1\right\}$.
\end{itemize}
The grid discretizes the sphere into layers of concentric shells, with each shell containing $n_S$ equally spaced points along the directions determined by the unit vectors. The discrete spherical uniform distribution places equal mass over each points of the grid, with $n_o/(n+1)$ mass on the origin and $1/(n+1)$ on the remaining points.
This ensures isotropic sampling at fixed radius onto $[0,1]$.

By definition of target distribution $\mathbb{U}_{n+1}$, it holds 
\begin{equation}\label{eq:distribution_empirical_transport}
\|T_{n+1}(Z_{n+1})\| \sim \mathbb{U}_{n+1} \left\{0, \frac{1}{n_R}, \frac{2}{n_R}, \ldots, 1\right\}.
\end{equation}

In order to define an empirical quantile region as \Cref{eq:exact_continuous_coverage}, we need an extrapolation $\bar{T}_{n+1}$ of $T_{n+1}$ out of the samples $(Z_i)_{i\in [n+1]}$. By definition of such maps $$\|\bar{T}_{n+1}(Z_{n+1})\|= \|T_{n+1}(Z_{n+1})\|$$ is still uniformly distributed. With an appropriate choice of radius $r_{\alpha, n+1}$, the empirical quantile region can be defined 
$$
\mathcal{R}_{\alpha, n+1} = \{z \in \mathbb{R}^d: \|\bar{T}_{n+1}(z)\| \leq r_{\alpha, n+1}\}.
$$
When working with such finite samples $Z_1, \ldots, Z_n, Z_{n+1}$, and considering the asymptotic regime \citep{chewi2024statistical, hallin2021}, the empirical source distribution $\mathbb{P}_{n+1}$ converges to the true distribution $\mathbb{P}$ and the empirical transport map $\bar T_{n+1}$ converges to the true transport map $T^\star$. As such, with the choice $r_{\alpha, n+1}= 1-\alpha$, one can expect that
$
\mathbb{P} \left(Z \in \mathcal{R}_{\alpha, n+1}\right)  \approx 1 - \alpha \text{ when } n \text{ is large}.
$

However, the core point of conformal prediction methodology is to go beyond asymptotic results or regularity assumptions about the data distribution. The following result show how to select a radius preserving the coverage with respect to the ground-truth distribution such as in \Cref{eq:transport_oracle_coverage}. 
\begin{proposition}
Given $n$ discrete sample points distributed over a sphere with radius $\{0, \frac{1}{n_R}, \frac{2}{n_R}, \ldots, 1\}$ and directions uniformly sampled on the sphere, the smallest radius to obtain a coverage $(1-\alpha)$ is 
determined by 
$$r_{\alpha, n+1} = \frac{j_\alpha}{n_R} \text{ where }
j_\alpha = \left\lceil \frac{(n+1) (1 - \alpha) - n_o}{n_S} \right\rceil,
$$
where $n_S$ is the number of directions, $n_R$ is the number of radius, and $n_o$ is the number of copies of the origin.
\end{proposition}

The corresponding conformal prediction set is obtained as:
\begin{equation}\label{eq:full_transport_otcp}
\{y \in \mathcal{Y}: \|\bar T_{n+1} \circ S(X_{n+1}, y)\| \leq r_{\alpha, n+1}\}.
\end{equation}

\begin{remark}[Computational Issues] \label{rm:computationally_infeasible}
While appealing, the previous result has notable computational limitations.
At every new candidate $y \in \mathcal{Y}$, the empirical transport map must be recomputed which might be untractable. Moreover, the coverage guarantee does not hold if the transport map is computed solely on a hold-out independent dataset, as it is usually done in split conformal prediction.
Plus, for computational efficiency, the empirical entropic map cannot be directly leveraged, since the target values would no longer follow a uniform distribution, as described in \Cref{eq:distribution_empirical_transport}.    
\end{remark}
To address these challenges, we propose two simple approaches in the following section.

\subsection{Optimal Transport Merging}

We introduce optimal transport merging, a procedure that reduces any vector-valued score $S(x, y) \in \mathbb{R}^d$ to a suitable 1D score using OT. We redefine the non-conformity score function of an observation as 
\begin{equation}\label{eq:otcp}
S_{\rm{OT-CP}}(x, y) =  \|T^\star \circ S(x, y)\|\end{equation}
where $T^\star$ is the optimal \citet{Bre91} map that pushes the distribution of vector-valued scores onto a uniform ball distribution $\mathbb{U}$ of the same dimension.
This approach ultimately relies on the natural ordering of the real line, making it possible to directly apply one-dimensional conformal prediction methods to the sequence of transformed scores 
$$Z_i = \|S_{\rm{OT-CP}}(X_i, Y_i)\| \text{ for } i \in [n+1] .$$

In practice, $T^\star$ can be replaced by any approximation $\hat T$ that preserves the permutation invariance of the score function.
The resulting conformal prediction set, \OTCP\, is
\begin{align*}
\mathcal{R}_{\mathrm{OT-CP}}(X_{n+1}, \alpha) &= \mathcal{R}_\alpha(\hat{T}, X_{n+1})
\end{align*}
with respect to a given transport map $\hat{T}$, and where
$$
\mathcal{R}_\alpha(\hat{T}, x) = 
\big\{y \in \mathcal{Y}: F_n(\|S_{\mathrm{OT-CP}}(x, y)\|_2)\leq 1-\alpha\big\}.
$$
have a coverage $(1-\alpha)$, where $F_n$ is empirical (univariate) cumulative distribution function of the observed scores $$\big\{\|S_{\rm{OT-CP}}(X_1, Y_1)\|, \ldots, \|S_{\rm{OT-CP}}(X_n, Y_n)\|\big\}.$$

\Cref{prop:Univariate_Conformal_prediction_Coverage} directly implies
$$
\mathbb{P}(Y_{n+1} \in \mathcal{R}_{\mathrm{OT-CP}}(X_{n+1})) \geq 1-\alpha.
$$

\begin{remark}
Our proposed conformal prediction framework \OTCP\, with optimal transport merging score function generalizes the \MergeCP\, approaches. More specifically, under the additional assumption that we are transporting a source Gaussian (resp. uniform) distribution to a target Gaussian (resp. uniform) distribution, the transport map is affine \cite{gelbrich1990formula, muzellec2018generalizing} with a positive definite linear map term. This results in \Cref{eq:otcp} being equivalent to the Mahalanobis distance.
\end{remark}

\subsection{Coverage Guarantees under Approximations}\label{sec:coverage}

When dealing with high-dimensional data or complex distributions, it is essential to find computationally feasible methods to approximate the optimal transport map $T^\star$ with a map $\hat{T}$. In practical applications, we will rely on empirical approximations of the \citet{Bre91} map using finite samples. Note that this approach may encouter a few statistical roadblocks, as such estimators are significantly hindered by the curse of dimensionality \citep{chewi2024statistical}.
However, conformal prediction allows us to maintain a coverage level irrespective of sample size limitations. We defer the presentation of this practical approach to section~\ref{subsec:entropic} and focus first on coverage guarantees.

\paragraph{Coverages of Approximated Quantile Region} \quad \\
Let us assume an arbitrary approximation $\hat T$ of the \citet{Bre91} map and define the corresponding quantile region as
$$
\mathcal{R}\big(\hat T, r\big) = \{z \in \mathbb{R}^d: \|\hat T(z)\| \leq r\},
$$
The coverage in \Cref{eq:transport_oracle_coverage} is not automatically maintained since $ \hat{\mathbb{U}} := \hat{T}_\# \mathbb{P} $ may not coincide with $ \mathbb{U} $. As a result, the validity of the approximated quantile region may be compromised unless we can control the magnitude of the error $ \|\hat{\mathbb{U}} - \mathbb{U}\| $, which requires additional regularity assumptions.  
In its standard formulation, conformal prediction relies on an empirical setting and does not directly apply to the continuous case, and hence does not provide a solution for calibrating entropic quantile regions. However, a careful inspection of the 1D case reveals that understanding the distribution of the probability integral transform is key:
\begin{itemize}
    \item $\mathbb{U}\left(\big\{0, \frac1n, \frac12, \ldots, 1\big\}\right) \sim F_n(Z) \neq F(Z) \sim \mathbb{U}(0, 1)$ .
\end{itemize}
Instead of relying on an analysis of approximation error to quantify the deviation $|F_n - F|$ under certain regularity conditions, conformal prediction fully characterizes the distribution of the probability integral transform and calibrates the radius of the quantile region accordingly.
We follow this idea and note that by definition, we have
$$
\mathbb{P}(\mathcal{R}(\hat T, r)) = \mathbb{P}(\|\hat T(z)\| \leq r) = \hat{ \mathbb{U}} (B(0, r)).
$$
Instead of relying on $\hat{ \mathbb{U}} \approx \mathbb{U}$, we define
\begin{equation}\label{eq:generic_radius}
r_\alpha(\hat T, \mathbb{P}) = \inf\{r: \hat{ \mathbb{U}} (B(0, r)) \geq 1-\alpha\}
\end{equation}
that naturally leads to a desired coverage with the approximated transported map. For $\hat{r}_\alpha = r_\alpha(\hat T, \mathbb{P})$, it holds
$$\mathbb{P}\left(Z \in \mathcal{R}(\hat T, \hat{r}_\alpha)\right) \geq 1-\alpha.$$

By extension, a quantile region of the vector-valued score $Z = S(X, Y) \in \mathbb{R}^d$ of a prediction model $\hat y$ provides an uncertainty set for the response of a given input $X$, with the prescribed coverage $(1-\alpha)$ expressed as
\begin{equation*}
\widehat{\mathcal{R}}_{\alpha}(X) = \big\{y \in \mathcal{Y}: \|\hat T \circ S(X, y)\| \leq \hat{r}_\alpha\big\}.
\end{equation*}
\begin{equation}\label{eq:transport_oracle_coverage} 
\mathbb{P}(Y \in \widehat{\mathcal{R}}_{\alpha}(X)) \geq 1 - \alpha.
\end{equation}

We give the finite sample analogy of \Cref{eq:transport_oracle_coverage}, which provides a coverage guarantee even when the transport map is an approximation obtained using both entropic regularization and finite sample data e.g in \Cref{eq:empirical_entropic_map}.\newline
Given such an approximated map $\hat T_{n+1}$ and applying
and the empirical radius $\hat r_{\alpha, n+1} = r_\alpha(\hat T_{n+1}, \mathbb{P}_{n+1})$, it holds\looseness=-1
$$
\mathbb{P}_{n+1}(Z_{n+1} \in \mathcal{R}(\hat T_{n+1}, \hat r_{\alpha, n+1})) \geq 1-\alpha.
$$
However, this is \emph{only} an empirical coverage statement:
$$
\frac{1}{n+1}\sum_{i=1}^{n+1} \mathds{1}\{Z_i \in \mathcal{R}(\hat T_{n+1}, \hat r_{\alpha, n+1})\} \geq 1 - \alpha
$$
which does not imply coverage wrt $\mathbb{P}$ unless $n \to \infty$. The following result shows how to obtain finite sample validity.

\begin{lemma}[Coverage of Empirical Quantile Region]\label{lm:coverage_empirical_measure}
Let $Z_1, \ldots, Z_n, Z_{n+1}$ be a sequence of exchangeable variables in $\mathbb{R}^d$, then,
$
   \mathbb{P}(Z_{n+1} \in \widehat {\mathcal{R}}_{\alpha, n+1}) \geq 1-\alpha,
$
where, for simplicity, we denoted the approximated empirical quantile region as $\widehat {\mathcal{R}}_{\alpha, n+1} = \mathcal{R}(\hat T_{n+1}, \hat r_{\alpha, n+1})$. 
\end{lemma}

This can be directly applied to obtain conformal prediction set for vector-valued non-conformity score functions 
$Z_i = S(X_i, Y_i) \in \mathbb{R}^d$ for $i$ in $[n+1]$ in \Cref{lm:coverage_empirical_measure}.

\begin{proposition}\label{prop:Vector_PIT_Guarantee}
The conformal prediction set is defined as
\begin{align*}
\widehat{\mathcal{R}}_{\alpha, n+1}(X_{n+1}) &= \left\{y \in \mathcal{Y} : \|\hat T \circ S(X_{n+1}, y)\| \leq \hat{r}_{\alpha, n+1}\right\} 
\end{align*}
with  
$\hat{r}_{\alpha,n+1} = \inf\big\{r \geq 0: \hat{\mathbb{U}}_{n+1} (B(0, r)) \geq 1-\alpha\big\}$.
It satisfies a distribution-free finite sample coverage guarantee
\begin{equation}\label{eq:valid_empirical_radius}
\mathbb{P}\left(Y_{n+1} \in \widehat{\mathcal{R}}_{\alpha, n+1}(X_{n+1})\right) \geq 1 - \alpha.
\end{equation}
\end{proposition}

Approaches relying on vector-valued probability integral transform, e.g., by leveraging Copulas, have been recently explored 
\cite{messoudi2021copula, park2024semiparametric} and concluded that loss of coverage can occur when the estimated copula of the scores deviates from the true copula and thus does not
formally guarantee finite-sample validity. To our knowledge, \Cref{prop:Vector_PIT_Guarantee} provides the first calibration guarantee for such confidence regions without assumptions on the distribution, for any approximation map $\hat T$.

\begin{figure*}[ht]
    \centering
    \includegraphics[width=\linewidth]{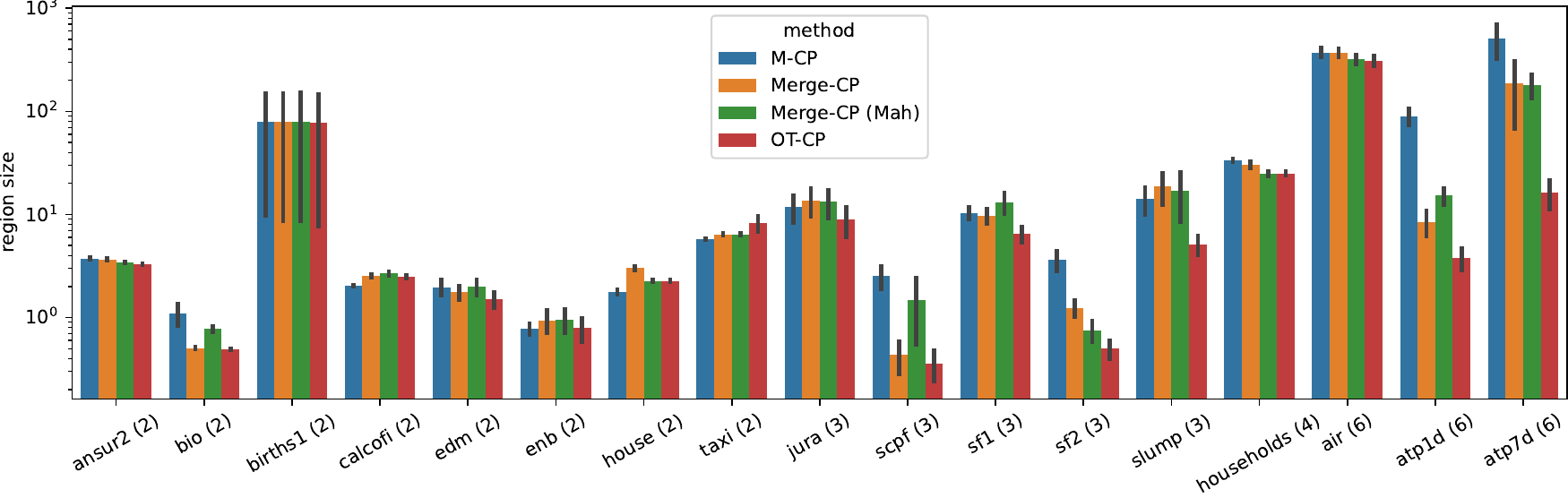}\vskip-.3cm
    \caption{We report the mean and standard error of the region size across 10 different seeds. For \MCP, we use $300$ samples to compute the conditional mean, and for \OTCP, we use $\varepsilon = 0.1$ and $2^{15}=32768$ points in the uniform target measure. Overall, \OTCP\ displays smaller region size than other baselines (13 out of 17 datasets). The output dimension $d$ of each dataset is provided next to its name.}
    \label{fig:small-region}
\end{figure*}

\subsection{Implementation with the Entropic Map}\label{subsec:entropic}

We assume access to two families of samples: residuals $(z_1,\dots ,z_n)$, and a discretization of the uniform grid on the sphere, $(u_1, \dots, u_m)$, with sizes $n,m$ that will be usually different, $n \ne m$.  Learning the entropic map estimator as in \Cref{subsec:entropic} requires running the \citet{Sinkhorn64} algorithm for a given regularization $\varepsilon$ on a $n\times m$ cost matrix. At test time, for each evaluation, computing the weights in \Cref{eq:gibbs} requires computing the distances of a new score $z$ to the uniform grid. The complexity is therefore $O(nm)$ when training the map and conformalizing its norms, and $O(m)$ to transport a conformity score for a given $y$. \looseness=-1

\textbf{Sampling on the sphere.}
As mentioned by \citet{hallin2021}, it is preferable to sample the uniform measure $\mathbb{U}_d$ with diverse samples. This can be achieved using stratified sampling on radii lengths and low-discrepancy samples picked on the sphere. We borrow inspiration from the review provided in \citep{nguyenquasi} and pick their \textit{Gaussian based} mapping approach \citep{basu2016quasi}. This consists %
of mapping a low-discrepancy sequence $w_1,\ldots,w_L$ on $[0,1]^d$ to a potentially low-discrepancy sequence $\theta_1,\ldots,\theta_L$ on $\mathbb{S}^{d-1}$ through the mapping $\theta= \Phi^{-1}(w)/\|\Phi^{-1}(w)\|_2$, where $\Phi^{-1}$ is the inverse CDF of $\mathcal{N}(0,1)$ applied entry-wise.\looseness=-1

\section{Experiments}\label{sec:experiments}
\subsection{Setup and Metrics}
We borrow the experimental setting provided by \citet{dheur2025multioutputconformalregressionunified} and benchmark multivariate conformal methods on a total of 24 tabular datasets. Total data size \( n \) in these datasets ranges from 103 to 50,000, with input dimension \( p \) ranging from 1 to 348, and output dimension \( d \) ranging from 2 to 16. We adopt their approach, which is to rely on a multivariate quantile function forecaster \citep[MQF$^2$,][]{kan2022multivariate}, a normalizing flow that is able to quantify output uncertainty conditioned on input $x$. However, in accordance with our stance mentioned in the background section, we will only assume access to the conditional mean (point-wise) estimator for \OTCP.
\begin{figure*}
    \centering
    \includegraphics[width=\linewidth]{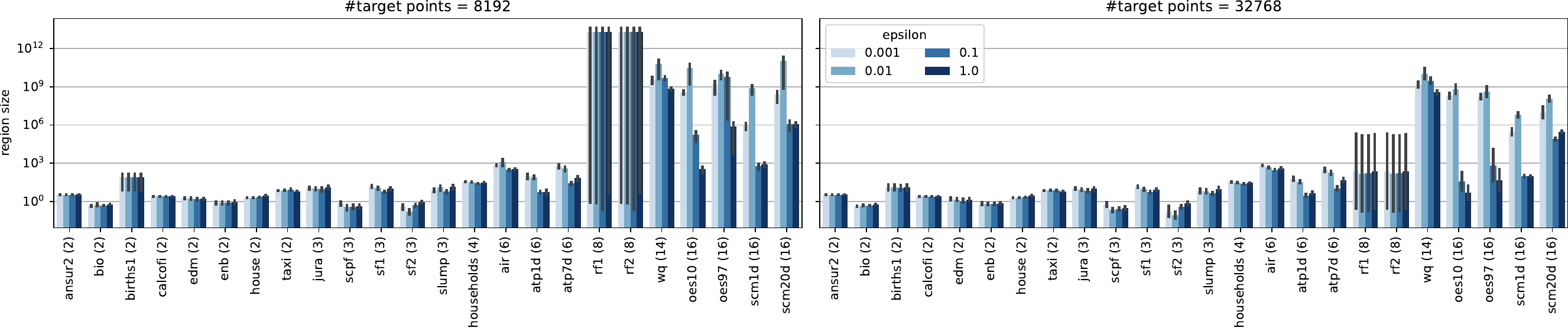}
    \vskip-.2cm
    \caption{This plot details the impact of the two important hyperparameters one needs to set in \OTCP: number of target points $m$ sampled from the uniform ball and the $\varepsilon$ regularization level. As can be seen, larger sample size $m$ improves region size (smaller the better) for roughly all datasets and regularization strengths. On the other hand, one must tune $\varepsilon$ to operate at a suitable regime: not too low, which results in the well-documented poor statistical performance of unregularized / linear program OT, nor too high, which would lead to a collapse of the entropic map to the sphere. Using OTT-JAX and its automatic normalizations, we see that $\varepsilon=0.1$ works best overall.}
    \label{fig:sup-region}
\end{figure*}

As is common in the field, we evaluate the methods using several metrics, including marginal coverage (MC), and mean region size (Size). The latter is using importance sampling, leveraging (when computing test time metrics only), the generative flexibility provided by the MQF$^2$ as an invertible flow. See \citep{dheur2025multioutputconformalregressionunified} and their code for more details on the experimental setup.
\begin{figure}
    \centering
    \includegraphics[width=\linewidth]{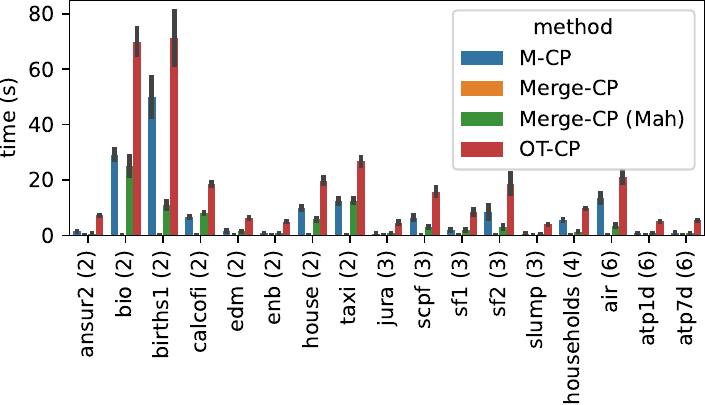}\vskip-.2cm
    \caption{Computational time on small dimensional datasets. \OTCP\ incurs more compute time due to the OT map estimation. See Fig.\ref{fig:big-time} for a similar picture for higher dimensional datasets.}
    \label{fig:small-time}
\end{figure}

\subsection{Hyperparameter Choices}
We apply default parameters for all three competing methods, \MCP\ and \MergeCP, using (or not) the Mahalanobis correction. For \MCP\ using conformalized quantile regression boxes, we follow \citep{dheur2025multioutputconformalregressionunified} and leverage the empirical quantiles return by MQF$^2$ to compute boxes \citep{zhou2024conformalized}.

\OTCP:\ our implementation requires tuning two important hyperparameters: the entropic regularization $\varepsilon$ and the total number of points used to discretize the sphere $m$, not necessarily equal to the input data sample size $n$. These two parameters describe a fundamental statistical and computational trade-off. On the one hand, it is known that increasing $m$ will mechanically improve the ability of $T_\varepsilon$ to recover in the limit $T^\star$ (or at least solve the semi-discrete~\citep{PeyCut19} problem of mapping $n$ data points to the sphere). However, large $m$ incurs a heavier computational price when running the \citeauthor{Sinkhorn64} algorithm. On the other hand, increasing $\varepsilon$ improves on \textit{both} computational and statistical aspects, but deviates further the estimated map from the ground truth $T^\star$ to target instead a blurred map. We have experimented with these aspects and derive from our experiments that both $m$ and $\varepsilon$ should be increased to track increase in dimension. As a sidenote, we do observe that debiasing the outputs of the \citeauthor{Sinkhorn64} algorithm does not result in improved results, which agrees with the findings in~\citep{pooladian2022debiaser}. We use the OTT-JAX toolbox~\citep{cuturi2022optimaltransporttoolsott} to compute these maps.

\subsection{Results}
We present results by differentiating datasets with small dimension $d\leq 6$ from datasets with higher dimensionality $14\leq d\leq 16$, that we expect to be more challenging to handle with OT approaches, owing to the curse of dimensionality that might degrade the quality of multivariate quantiles. Results in \Cref{fig:big-region} indicate an improvement (smaller region for similar coverage) on 15 out of 18 datasets in lower dimensions, this edge vanishing in the higher-dimensional regime. Ablations provided in Figure~\ref{fig:sup-region} highlight the role of $\varepsilon$ and $m$, the entropic regularization strength and the sphere size respectively. These results show that results for high $m$ tend to be better but more costly, while the tuning of the regularization strength $\varepsilon$ needs to be tuned according to dimension~\citep{vacher2022parameter}. Finally, \Cref{fig:taxi} provides an illustration of the non-elliptic CP regions outputted by \OTCP, by pulling back the rescaled uniform sphere using the inverse entropic mapping described in Section~\ref{subsec:entropic}.

\section{Conclusion}
We have proposed \OTCP, a new approach that can leverage a recently proposed formulation for multivariate quantiles that uses optimal transport theory and optimal transport map estimators. We show the theoretical soundness of this approach, but, most importantly, demonstrate its applicability throughout a broad range of tasks compiled by \citep{dheur2025multioutputconformalregressionunified}. Compared to similar baselines that either use a conditional mean regression estimator (\MergeCP), or more involved quantile regression estimators (\MCP), \OTCP\, shows overall superior performance, while incurring, predictably, a higher train / calibration time cost. The challenges brought forward by the estimation of OT maps in high dimensions~\citep{chewi2024statistical} require being particularly careful when tuning entropic regularization and grid size. However, we show that there exists a reasonable setting for both of these parameters that delivers good performance across most tasks.

\begin{figure}
    \centering
    \includegraphics[width=\linewidth]{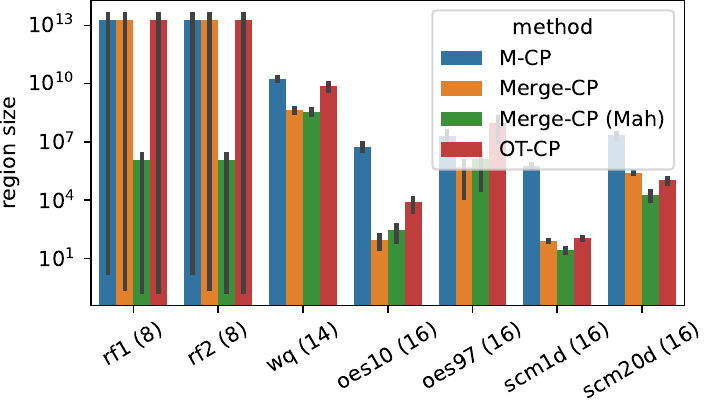}
    \vskip-.5cm
    \caption{As in \ref{fig:small-region}, we report mean and standard errors for region size (log scale) across 10 different seeds for larger datasets. We keep the same parameters and importantly $\varepsilon = 0.1$ and $2^{15}=32768$ points in the uniform target measure. We expect the performance of \OTCP\, to decrease with dimensionality, but it does provide a convincing alternative to the other approaches.}
    \label{fig:big-region}
\end{figure}

\begin{figure}
    \centering
    \includegraphics[width=\linewidth]{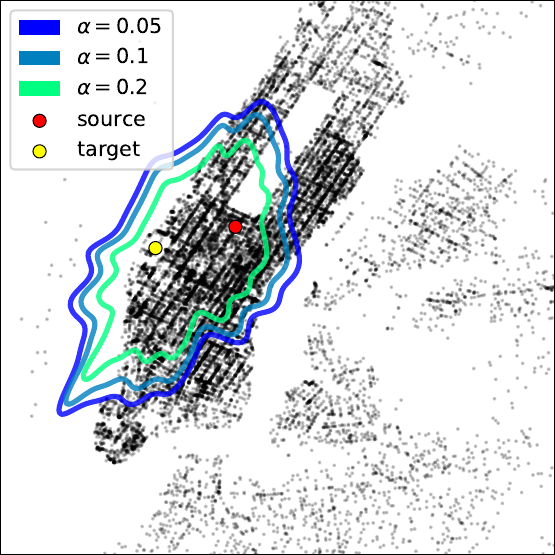}
    \caption{Conformal sets recovered by mapping back the reduced sphere on the Manhattan map, in agreement with Equation~\ref{eq:transport_oracle_coverage}, on a prediction for the \texttt{taxi} dataset. We use the inverse entropic map mentioned in Section~\ref{subsec:entropic}, mapping back the gridded sphere of size $m=2^{15}$ for each level, and plotting its outer contour.}
    \label{fig:taxi}
\end{figure}

\newpage

\section{Concurrent Work.}\label{sec:concurrent}
Concurrently to our work, \citet{thurin2025optimaltransportbasedconformalprediction} proposed recently to leverage OT in CP with a similar approach, deriving a similar CP set as in \Cref{eq:full_transport_otcp} and analyzing a variant with asymptotic conditional coverage under additional regularity assumptions.
However, our methods differ in several key aspects. On the computational side, our implementation leverages general entropic maps (Section~\ref{subsec:entropic}) without compromising finite-sample coverage guarantees, an aspect we analyze in detail in Section~\ref{sec:coverage}.
In contrast, their approach requires solving a linear assignment problem,  using for instance the Hungarian algorithm, which has cubic complexity $O(n^3)$ in the number of target points, and which also requires having a target set on the sphere that is of the same size as the number of input points. With our notations in Section~\ref{subsec:entropic}, they require $n=m$, whereas we set $m$ to anywhere between $2^{12}$ and $2^{15}$, independently of $n$. While they mention efficient approximations that reduce complexity to quadratic in \citep[Remark 2.3]{thurin2025optimaltransportbasedconformalprediction}, their theoretical results do not yet cover these cases since their analysis relies on the fact that ranks are random permutations of $\{1/n, 2/n, \ldots, 1\}$, which cannot be extended to using Sinkhorn with soft assignment.
In contrast, our work establishes formal theoretical coverage guarantees even when approximated (pre-trained) transport map are used.\looseness=-1

\bibliography{ref}
\bibliographystyle{icml2025}

\appendix
\onecolumn

\section{Appendix}
We provide a few additional results related to the experiments proposed in Section~\ref{sec:experiments}.
\begin{figure}
    \centering
    \includegraphics[width=.5\linewidth]{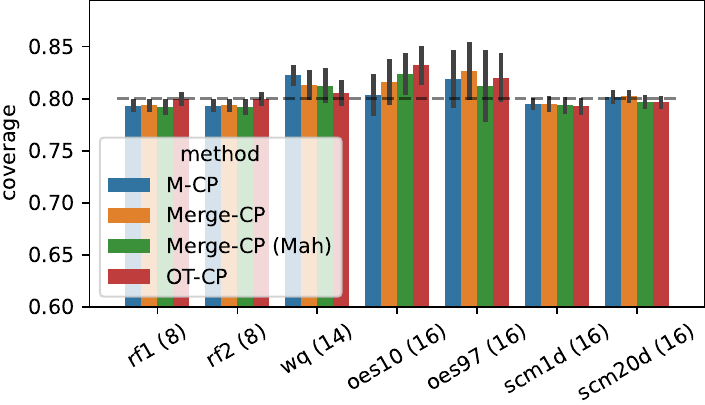}
    \caption{Coverage for higher dimensional datasets, corresponding to the setting displayed in \Cref{fig:big-coverage}.}
    \label{fig:big-coverage}
\end{figure}

\begin{figure}
    \centering
    \includegraphics[width=.5\linewidth]{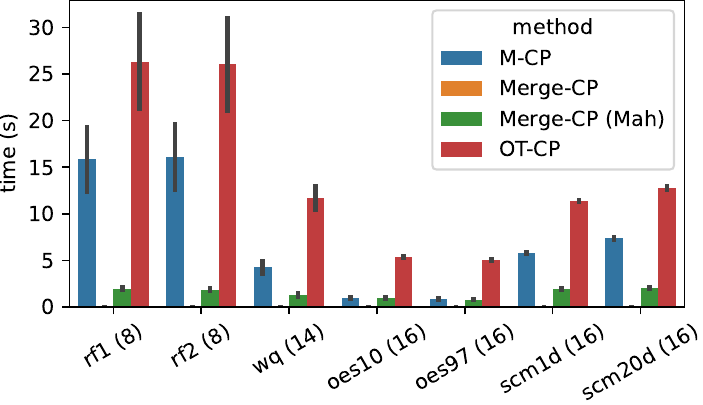}
    \caption{Runtimes for higher dimensional datasets, corresponding to the setting displayed in \Cref{fig:big-coverage}.}
    \label{fig:big-time}
\end{figure}

\begin{figure}
    \centering
    \includegraphics[width=\linewidth]{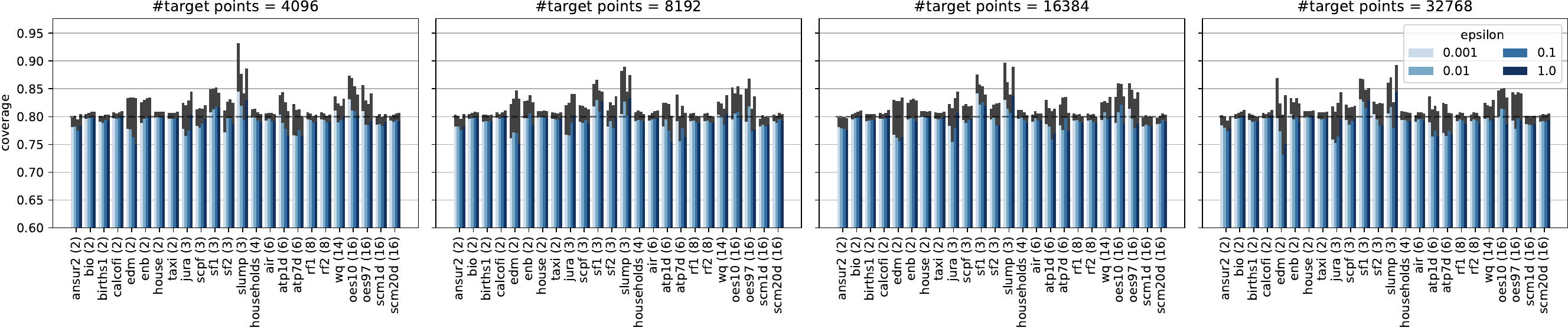}
    \caption{Ablation: coverage quality as a function of hyperparameters, with the setting corresponding to \Cref{fig:sup-region}.}
    \label{fig:sup-coverage}
\end{figure}
\begin{figure}
    \centering
    \includegraphics[width=\linewidth]{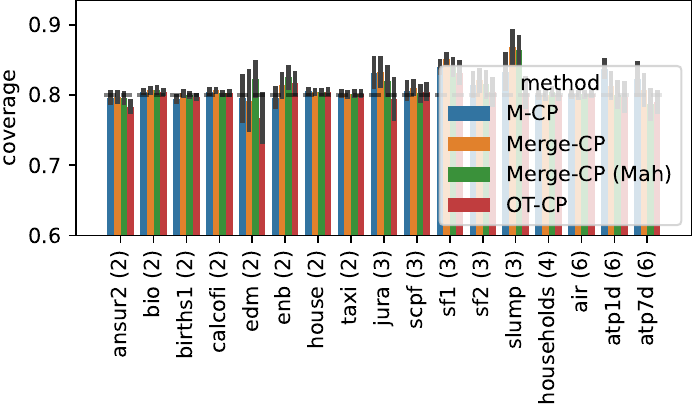}
    \caption{Coverage of all baselines on small dimensional datasets, corresponding to the region sizes given in \Cref{fig:small-region}.}
    \label{fig:small-coverage}
\end{figure}

\begin{figure*}
    \centering
    \includegraphics[width=\linewidth]{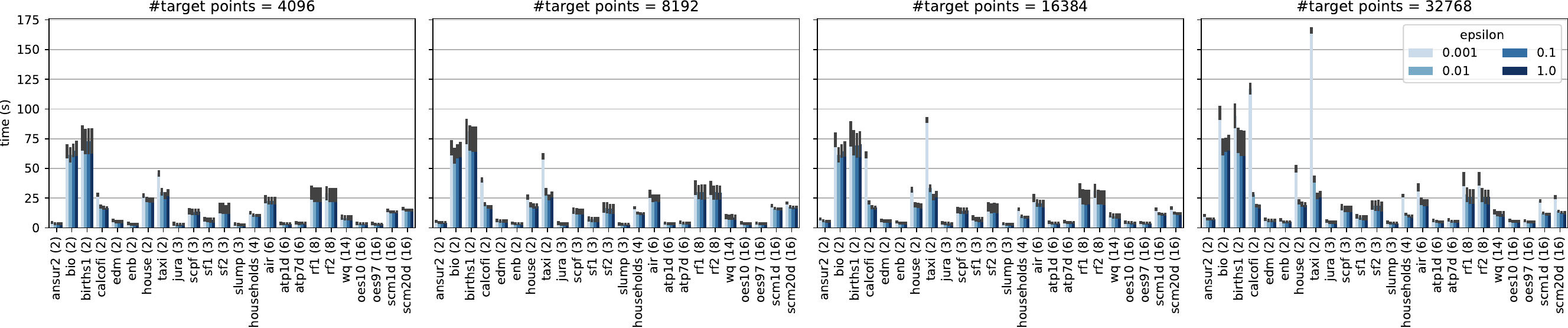}
    \caption{Ablation: running time as a function of hyperparameters, with the setting corresponding to \Cref{fig:sup-region}.}
    \label{fig:enter-label}
\end{figure*}

\newpage
\section{Proofs}

\begin{proposition}
Given $n$ discrete sample points distributed over a sphere with radii $\{0, \frac{1}{n_R}, \frac{2}{n_R}, \ldots, 1\}$ and directions uniformly sampled on the sphere, the smallest radius $r_\alpha = \frac{j_\alpha}{n_R}$ satisfying $(1-\alpha)$-coverage is
is determined by 
$$
j_\alpha = \left\lceil \frac{(n+1) (1 - \alpha) - n_o}{n_S} \right\rceil,
$$
where $n_S$ is the number of directions, $n_R$ is the number of radii, and $n_o$ is the number of copies of the origin ($\|U\| = 0$).
\end{proposition}

\begin{proof}
The discrete spherical uniform distribution places the same probability mass on all $n+1$ sample points, including the $n_o$ copies of the origin. As such, given a radius $r_j = \frac{j}{n_R}$, we have
$$
\mathbb{P}(\|U\| = r_j) = n_S \cdot \frac{1}{n+1}.
$$
The cumulative probability up to radius $r_j$ is given by:
\begin{align*}
\mathbb{P}(\|U\| \leq r_j) = \mathbb{P}(\|U\| = 0) + \sum_{k=1}^j \mathbb{P}(\|U\| = r_k) 
= \frac{n_o}{n+1} + j \times \frac{n_S}{n+1}.
\end{align*}
To find the smallest $r_\alpha = \frac{j_\alpha}{n_R}$ such that $\mathbb{P}(\|U\| \leq r_{j_\alpha}) \geq 1 - \alpha$, it suffices to solve:
$$
\frac{n_o}{n+1} + j_\alpha \times \frac{n_S}{n+1} \geq 1 - \alpha.
$$
\end{proof}

\begin{lemma}[Coverage of Empirical Quantile Region]
Let $Z_1, \ldots, Z_n, Z_{n+1}$ be a sequence of exchangeable variables in $\mathbb{R}^d$, then,
$
   \mathbb{P}(Z_{n+1} \in \widehat {\mathcal{R}}_{\alpha, n+1}) \geq 1-\alpha,
$
where, for simplicity, we denoted the approximated empirical quantile region as $\widehat {\mathcal{R}}_{\alpha, n+1} = \mathcal{R}(\hat T_{n+1}, \hat r_{\alpha, n+1})$. 
\end{lemma}
 
\begin{proof}
By exchangeability of $Z_1, \ldots, Z_{n+1}$ and symmetry of the set $\widehat{\mathcal{R}}_{\alpha, n+1}$, it holds
$$
    \mathbb{P}(Z_{n+1} \in \widehat{\mathcal{R}}_{\alpha, n+1}) = \mathbb{P}(Z_{i} \in \widehat{\mathcal{R}}_{\alpha, n+1}) \qquad \forall i \in [n+1].
$$
By taking the average on both side, we have: 
\begin{align*}
    \mathbb{P}(Z_{n+1} \in \widehat{\mathcal{R}}_{\alpha, n+1}) &= \frac{1}{n+1}\sum_{i=1}^{n+1} \mathbb{P}(Z_{i} \in \widehat{\mathcal{R}}_{\alpha, n+1}) \\
    &= \mathbb{E} \left[ \frac{1}{n+1}\sum_{i=1}^{n+1} \mathds{1}\{Z_i \in \widehat{\mathcal{R}}_{\alpha, n+1}\}  \right] \\
    &= \mathbb{E} \bigg[ \mathbb{P}_{n+1}(Z_{n+1} \in \widehat{\mathcal{R}}_{\alpha, n+1}) \bigg]  \\
    &\geq 1 - \alpha.
\end{align*}
\end{proof}

\begin{table}
\centering
\resizebox{0.9\textwidth}{!}{%
\begin{tabular}{llllllllllllll}
\toprule
 &  & \multicolumn{12}{r}{} \\
 &  & ansur2 (2) & bio (2) & births1 (2) & calcofi (2) & edm (2) & enb (2) & house (2) & taxi (2) & jura (3) & scpf (3) & sf1 (3) & sf2 (3) \\
epsilon & \#target &  &  &  &  &  &  &  &  &  &  &  &  \\
\midrule
\multirow[t]{4}{*}{0.001} & 4096 & 3.3±0.064 & 0.46±0.057 & 78±70 & 2.6±0.089 & 1.9±0.3 & 0.81±0.21 & 2±0.051 & 7±0.12 & 13±2.6 & 0.78±0.4 & 14±2.6 & 0.82±0.32 \\
 & 8192 & 3.4±0.059 & 0.45±0.057 & 78±70 & 2.6±0.089 & 1.9±0.29 & 0.81±0.2 & 2±0.05 & 7±0.13 & 11±2.6 & 0.73±0.23 & 16±3.9 & 0.4±0.16 \\
 & 16384 & 3.4±0.059 & 0.46±0.058 & 78±70 & 2.6±0.093 & 1.8±0.28 & 0.83±0.21 & 2±0.048 & 7±0.13 & 12±2.3 & 0.87±0.34 & 21±4.8 & 0.44±0.2 \\
 & 32768 & 3.4±0.063 & 0.46±0.058 & 78±70 & 2.6±0.092 & 1.9±0.3 & 0.81±0.2 & 2±0.05 & 7±0.13 & 12±2.6 & 1.2±0.47 & 16±2.9 & 0.57±0.18 \\
\cline{1-14}
\multirow[t]{4}{*}{0.01} & 4096 & 3.3±0.055 & 0.55±0.12 & 78±70 & 2.5±0.084 & 1.9±0.3 & 0.81±0.21 & 2±0.05 & 7.5±0.63 & 11±2.8 & 0.43±0.15 & 12±2.1 & 0.2±0.086 \\
 & 8192 & 3.3±0.054 & 0.56±0.13 & 78±70 & 2.5±0.082 & 1.8±0.3 & 0.8±0.21 & 2±0.049 & 7.5±0.69 & 10±2.6 & 0.37±0.15 & 12±2.8 & 0.17±0.063 \\
 & 16384 & 3.3±0.045 & 0.56±0.12 & 78±70 & 2.5±0.082 & 1.7±0.24 & 0.8±0.21 & 2±0.05 & 7.5±0.71 & 13±4.3 & 0.4±0.18 & 11±2.9 & 0.19±0.076 \\
 & 32768 & 3.3±0.064 & 0.56±0.12 & 78±70 & 2.5±0.085 & 1.7±0.26 & 0.82±0.22 & 2±0.049 & 7.5±0.69 & 10±2.7 & 0.41±0.17 & 12±2.6 & 0.18±0.071 \\
\cline{1-14}
\multirow[t]{4}{*}{0.1} & 4096 & 3.3±0.058 & 0.49±0.011 & 78±70 & 2.5±0.084 & 1.6±0.25 & 0.81±0.21 & 2.3±0.065 & 8.3±1.4 & 9.2±2.8 & 0.37±0.15 & 6.6±0.96 & 0.48±0.1 \\
 & 8192 & 3.3±0.059 & 0.49±0.011 & 78±70 & 2.5±0.084 & 1.6±0.26 & 0.8±0.21 & 2.3±0.065 & 8.2±1.5 & 9.4±2.9 & 0.4±0.15 & 6.1±0.89 & 0.53±0.11 \\
 & 16384 & 3.3±0.054 & 0.49±0.012 & 78±70 & 2.5±0.081 & 1.6±0.26 & 0.8±0.21 & 2.3±0.058 & 8.2±1.4 & 9.4±2.9 & 0.37±0.12 & 6.4±0.83 & 0.45±0.092 \\
 & 32768 & 3.3±0.051 & 0.49±0.011 & 77±70 & 2.5±0.083 & 1.5±0.25 & 0.79±0.2 & 2.3±0.057 & 8.2±1.4 & 8.9±2.9 & 0.36±0.12 & 6.5±1.2 & 0.5±0.1 \\
\cline{1-14}
\multirow[t]{4}{*}{1} & 4096 & 3.6±0.055 & 0.65±0.019 & 78±70 & 2.5±0.1 & 1.7±0.27 & 0.92±0.24 & 3±0.13 & 6.4±0.14 & 13±4 & 0.45±0.16 & 9.5±1.9 & 0.84±0.13 \\
 & 8192 & 3.6±0.067 & 0.59±0.013 & 78±70 & 2.5±0.099 & 1.7±0.26 & 0.91±0.24 & 3±0.14 & 6.3±0.14 & 13±4 & 0.42±0.14 & 10±1.8 & 0.93±0.16 \\
 & 16384 & 3.5±0.072 & 0.57±0.016 & 78±70 & 2.5±0.099 & 1.7±0.27 & 0.91±0.24 & 3±0.13 & 6.4±0.14 & 14±4 & 0.48±0.17 & 9.8±1.7 & 0.91±0.17 \\
 & 32768 & 3.5±0.061 & 0.6±0.028 & 78±71 & 2.5±0.1 & 1.7±0.27 & 0.91±0.24 & 2.9±0.13 & 6.4±0.15 & 13±4 & 0.47±0.17 & 10±1.7 & 0.9±0.17 \\
\cline{1-14}
\bottomrule
\end{tabular}
}
\end{table}

\begin{table}
\centering
\resizebox{0.9\textwidth}{!}{%
\begin{tabular}{lllllll}
\toprule
 &  & \multicolumn{5}{r}{} \\
 &  & slump (3) & households (4) & air (6) & atp1d (6) & atp7d (6) \\
epsilon & \#target &  &  &  &  &  \\
\midrule
\multirow[t]{4}{*}{0.001} & 4096 & 15±7.6 & 37±1.4 & 2.6E+03±1.9E+03 & 81±19 & 8.5E+02±4.5E+02 \\
 & 8192 & 7.9±2 & 36±1.9 & 7.1E+02±56 & 99±41 & 5.9E+02±1.8E+02 \\
 & 16384 & 11±3.7 & 34±1.3 & 6.9E+02±52 & 65±19 & 9.4E+02±3E+02 \\
 & 32768 & 12±4.3 & 36±2.6 & 6.8E+02±36 & 87±28 & 5.1E+02±2E+02 \\
\cline{1-7}
\multirow[t]{4}{*}{0.01} & 4096 & 20±6.8 & 37±1.6 & 8.5E+02±1E+02 & 85±24 & 7.9E+02±4.1E+02 \\
 & 8192 & 12±4.9 & 34±1.7 & 1.3E+03±7E+02 & 82±24 & 4E+02±1.5E+02 \\
 & 16384 & 7.1±2.2 & 33±0.81 & 5.5E+02±47 & 1.1E+02±26 & 3.7E+02±68 \\
 & 32768 & 10±4 & 31±0.97 & 4.8E+02±51 & 42±9.1 & 2.8E+02±98 \\
\cline{1-7}
\multirow[t]{4}{*}{0.1} & 4096 & 5.8±1.3 & 27±1.3 & 3.2E+02±32 & 8.1±1.7 & 33±9.2 \\
 & 8192 & 5.9±1.3 & 26±1.3 & 3.1E+02±33 & 5.7±1 & 27±6.9 \\
 & 16384 & 5.9±1.4 & 25±1 & 3.1E+02±34 & 4±1.4 & 26±7.7 \\
 & 32768 & 5.1±1.1 & 25±1 & 3.1E+02±34 & 3.8±0.88 & 16±5.1 \\
\cline{1-7}
\multirow[t]{4}{*}{1} & 4096 & 14±5.3 & 29±1.3 & 4.3E+02±31 & 6.2±1.7 & 69±25 \\
 & 8192 & 15±5.3 & 30±2.1 & 3.4E+02±38 & 5.6±2.2 & 69±25 \\
 & 16384 & 16±5.6 & 28±1.1 & 4.1E+02±36 & 6.1±2 & 76±27 \\
 & 32768 & 15±5.5 & 29±1.9 & 4.3E+02±38 & 5.6±1.5 & 73±24 \\
\cline{1-7}
\bottomrule
\end{tabular}
}
\end{table}

\begin{table}
\centering
\resizebox{0.9\textwidth}{!}{%
\begin{tabular}{lllllllll}
\toprule
 &  & \multicolumn{7}{r}{} \\
 &  & rf1 (8) & rf2 (8) & wq (14) & oes10 (16) & oes97 (16) & scm1d (16) & scm20d (16) \\
epsilon & \#target &  &  &  &  &  &  &  \\
\midrule
\multirow[t]{4}{*}{0.001} & 4096 & 2E+13±2E+13 & 2E+13±2E+13 & 7.1E+09±3E+09 & 2.9E+08±8.3E+07 & 8.7E+08±4E+08 & 4E+07±3.6E+07 & 1.7E+07±1.1E+07 \\
 & 8192 & 2E+13±2E+13 & 2E+13±2E+13 & 3.7E+09±1.9E+09 & 3.7E+08±1.3E+08 & 1.4E+09±1.2E+09 & 9.3E+05±5E+05 & 2.5E+08±1.9E+08 \\
 & 16384 & 2E+13±2E+13 & 2E+13±2E+13 & 6.6E+09±3.2E+09 & 5.6E+08±4.3E+08 & 2.5E+08±1.3E+08 & 3.5E+05±1.3E+05 & 8.9E+07±5.7E+07 \\
 & 32768 & 2E+13±2E+13 & 2E+13±2E+13 & 3.1E+09±1.2E+09 & 5.5E+08±3E+08 & 3.1E+08±9.5E+07 & 9.7E+05±4.5E+05 & 1.3E+09±1.3E+09 \\
\cline{1-9}
\multirow[t]{4}{*}{0.01} & 4096 & 2E+13±2E+13 & 2E+13±2E+13 & 1.1E+10±7.3E+09 & 4.3E+09±3.8E+09 & 3.5E+09±2.5E+09 & 4.1E+08±3.8E+08 & 1.3E+11±1.1E+11 \\
 & 8192 & 2E+13±2E+13 & 2E+13±2E+13 & 6.4E+10±6E+10 & 3E+10±2.8E+10 & 1E+10±6.1E+09 & 8.1E+08±5.5E+08 & 1.1E+11±1.1E+11 \\
 & 16384 & 2E+13±2E+13 & 2E+13±2E+13 & 3.3E+09±7.9E+08 & 1.1E+09±4.3E+08 & 1E+10±5.7E+09 & 4.8E+07±3.7E+07 & 1.3E+09±8.3E+08 \\
 & 32768 & 2E+13±2E+13 & 2E+13±2E+13 & 5.1E+11±4.9E+11 & 6.5E+09±5E+09 & 4E+09±3.2E+09 & 1.6E+07±9.5E+06 & 2.7E+08±1.3E+08 \\
\cline{1-9}
\multirow[t]{4}{*}{0.1} & 4096 & 2E+13±2E+13 & 2E+13±2E+13 & 8.7E+09±3.7E+09 & 4.8E+04±3.2E+04 & 6E+09±6E+09 & 1.5E+03±6.7E+02 & 1.3E+06±6.4E+05 \\
 & 8192 & 2E+13±2E+13 & 2E+13±2E+13 & 4.8E+09±1.5E+09 & 1.7E+05±1.3E+05 & 6E+09±6E+09 & 6.2E+02±2.8E+02 & 1.2E+06±8.7E+05 \\
 & 16384 & 2E+13±2E+13 & 2E+13±2E+13 & 1.3E+10±6.8E+09 & 5.2E+04±4.7E+04 & 5.6E+09±5.6E+09 & 2.2E+02±46 & 2.9E+05±1E+05 \\
 & 32768 & 2E+13±2E+13 & 2E+13±2E+13 & 7.4E+09±2.9E+09 & 7.6E+03±5.1E+03 & 9.2E+07±8.1E+07 & 1.1E+02±17 & 1.1E+05±3.1E+04 \\
\cline{1-9}
\multirow[t]{4}{*}{1} & 4096 & 2E+13±2E+13 & 2E+13±2E+13 & 8E+08±2E+08 & 6.6E+02±3.4E+02 & 8.3E+05±8.1E+05 & 4.1E+02±76 & 5.2E+05±6.5E+04 \\
 & 8192 & 2E+13±2E+13 & 2E+13±2E+13 & 6.9E+08±1.7E+08 & 3.5E+02±1.8E+02 & 7.7E+05±7.6E+05 & 8.5E+02±3.1E+02 & 1.1E+06±3.9E+05 \\
 & 16384 & 2E+13±2E+13 & 2E+13±2E+13 & 5.3E+08±1.2E+08 & 2.2E+02±1.5E+02 & 4E+05±4E+05 & 1.3E+02±14 & 4.7E+05±1.8E+05 \\
 & 32768 & 2E+13±2E+13 & 2E+13±2E+13 & 5.5E+08±1.5E+08 & 1.9E+02±1.6E+02 & 3.1E+05±3.1E+05 & 1E+02±11 & 3.4E+05±6.4E+04 \\
\cline{1-9}
\bottomrule
\end{tabular}
}
\caption{Mean region size for varying $\varepsilon$ and the number of target points in the ball.}
\end{table}

\end{document}